\documentclass[conference]{IEEEtran}

\usepackage{amsmath,amssymb,amsthm,mathtools,bm}
\usepackage{graphicx}
\usepackage{booktabs}
\usepackage{balance}
\usepackage{placeins}
\usepackage{microtype}
\usepackage{cite}
\usepackage[hidelinks]{hyperref}
\usepackage{orcidlink}
\usepackage{fontawesome5}

\newtheorem{lemma}{Lemma}
\newtheorem{theorem}{Theorem}
\newtheorem{proposition}{Proposition}
\newtheorem{corollary}{Corollary}
\newtheorem{remark}{Remark}
\newcommand{\E}{\mathbb{E}}
\newcommand{\Var}{\mathrm{Var}}
\newcommand{\R}{\mathbb{R}}

\title{Jacobian-Velocity Bounds for Deployment Risk\\Under Covariate Drift}
\author{
\IEEEauthorblockN{Jonathan R. Landers}
\IEEEauthorblockA{
\orcidlink{0000-0003-1872-6179}\,\href{https://orcid.org/0000-0003-1872-6179}{0000-0003-1872-6179}\\
New York, NY, USA\\
\href{mailto:jonathan.robert.landers@gmail.com}{jonathan.robert.landers@gmail.com}\\[-1pt]
{\footnotesize
\href{https://github.com/jonland82/jacobian-velocity-bounds}{\faGithub\ Repository}
\quad
\href{https://jonland82.github.io/jacobian-velocity-bounds/}{\faGlobe\ Project site}}
}
}

\begin{document}
\maketitle

\begin{abstract}
We study long-horizon deployment of a frozen predictor under dynamic covariate shift. A time-domain Poincar\'e inequality first reduces temporal risk volatility to derivative energy, and a Jacobian-velocity theorem supplies the corresponding pathwise control. A remainder extension separates the model-mediated, controllable term from conditional-risk variation not transmitted through the score Jacobian. Under low-rank drift, the bound decomposes the controllable term into tangent energy parallel and orthogonal to the drift subspace. This motivates drift-aligned tangent regularization (DTR), its anisotropic extension with a weaker orthogonal guardrail, and a matched monitoring proxy. Controlled experiments verify the temporal inequality, the efficiency of directional smoothing, and the necessity of the remainder term. The main field study uses a 128-dimensional gas-sensor array over ten chronological batches. A rank-$4$ basis estimated entirely before deployment beats ten matched random subspaces, and validation-selected anisotropic DTR improves deployment cross-entropy and volatility over pure DTR in all ten matched seeds while retaining substantially more accuracy than isotropic smoothing. Air Quality and power-consumption studies provide complementary regression evidence. Together, the results support a precise design rule: concentrate smoothness along expected motion, and use limited orthogonal smoothing when the estimated drift subspace may be incomplete.
\end{abstract}


\section{Introduction}
Distribution shift in data encountered by a predictive model is often studied statically. Examples include source-target generalization gaps in domain adaptation \cite{ben-david-shift} and dataset-shift taxonomies that separate covariate, label, and concept shift \cite{moreno-shift}. Long-horizon deployment is different. A model may be frozen while the environment moves continuously, so the relevant object is not a single out-of-domain error but a risk trajectory. In evolving or temporally structured settings, performance can move materially over time \cite{gama-drift,koh-wilds,yao-wildtime,han-temporal}. The practical question is therefore not only whether shift occurs, but how rapidly deployment risk can vary once deployment begins.

This paper develops a geometric answer. If the feature stream follows a path $X_t$ and the predictor has local tangent map $J_f(X_t)$, then instability is governed not by shift magnitude alone but by the directional interaction between the data velocity $\dot X_t$ and the model tangent geometry. The dangerous quantity is $J_f(X_t)\dot X_t$. The intuition is simple. The same deployment path can be almost harmless for one model and damaging for another, because only motion through steep tangent directions accumulates into large score variation. If the stream moves through a flat direction, drift can be present without being operationally dangerous. If it moves through a steep direction, even gentle but persistent motion can compound into long-horizon degradation. Under explicit along-path regularity and directional domination assumptions, we establish a quadratic upper bound on temporal risk volatility in terms of the Jacobian--velocity interaction:
\[
\Var_U(r(U)) \le \frac{\beta^2 T}{\pi^2}\int_0^T \E\|J_f(X_t)\dot X_t\|^2dt.
\]

Here $r(t)$ is expected deployment risk, $U\sim\mathrm{Unif}[0,T]$, $T$ is the deployment horizon, and $\beta$ is the directional-domination constant specified below.

The paper makes three contributions. First, a time-domain Poincar\'e inequality and a Jacobian-velocity theorem identify directional tangent energy as the model-mediated quantity governing frozen-model deployment volatility. A remainder extension separates that actionable component from conditional-risk variation outside the score geometry. Second, the resulting low-rank decomposition prescribes an anisotropic drift-aligned tangent regularization (DTR) family that includes pure DTR and isotropic smoothing as endpoints, with a matched monitoring proxy as a secondary consequence of the same geometry. Third, controlled and real-data experiments test the complete chain. The gas-sensor study is the central deployment test: it uses a strict predeployment rank-$4$ basis in $128$ dimensions, multiple random-subspace controls, and a two-parameter regularization sweep. Synthetic studies isolate the theorem, alignment, and remainder mechanisms, while Air Quality and Tetouan test transfer to regression settings.

The paper is scoped to frozen-model deployment under covariate drift. It does not address general distribution shift or test-time adaptation, but asks specifically what local geometry governs frozen-model degradation and whether that geometry remains visible on real drifting data.
Code, scripts, and supplementary materials are available in a public repository \cite{landers-repo}. A companion project page with figures and summaries is available at \cite{landers-site}.

\section{Related Work}
This work sits closest to temporal distribution shift and dynamic deployment evaluation. Classical shift analyses study source-target error under changed input distributions \cite{ben-david-shift}, dataset-shift taxonomies separate covariate, label, and concept shift \cite{moreno-shift}, and concept-drift work emphasizes sequential adaptation under evolving streams \cite{gama-drift}. More recent benchmark-driven work has made temporal structure explicit, both in naturally shifted datasets and in stylized train-deploy protocols \cite{koh-wilds,yao-wildtime,han-temporal}. Related theory also studies heterogeneous distribution shift without an explicit temporal deployment path \cite{simchowitz-shift}. Here the focus is the volatility of one frozen model along a deployment path rather than a worst-case shift gap.

This work is also adjacent to test-time adaptation \cite{sun-ttt,wang-tent,gui-atta}. Those methods update the model online to track changing data. Here the predictor remains frozen, isolating the geometric mechanism of degradation and yielding a quantity that can be analyzed, regularized, and monitored without mixing in the dynamics of online parameter updates.

Finally, the work draws on two technical lines. Regularization relates to Jacobian-based robustness \cite{sokolic-jacobian,kim-jacobian} and sensitivity-generalization evidence \cite{novak-sensitivity}. Tangent Prop penalizes derivatives along prescribed transformation directions \cite{simard-tangentprop}. The Manifold Tangent Classifier learns local manifold tangents for a similar penalty \cite{rifai-manifold-tangent}. Directional Jacobian regularization has also stabilized long-term integration of learned neural differential equations \cite{janvier-neural-de}. Here the directions represent anticipated covariate motion for a frozen predictor, linking the penalty to temporal risk volatility. Anisotropic DTR retains a smaller orthogonal penalty against residual motion. Monitoring relates to shift detection, uncertainty under shift, and sequential change detection \cite{rabanser-shift,ovadia-shift,wu-rscusum,cao-xie-allerton2015,xie-lowrank-allerton2016,krishnamurthy-allerton2022,cooper-allerton2024}. Our hazard score matches the drift geometry used in training rather than claiming sequential optimality.

\begin{figure*}[t]
  \centering
  \includegraphics[width=0.9\textwidth]{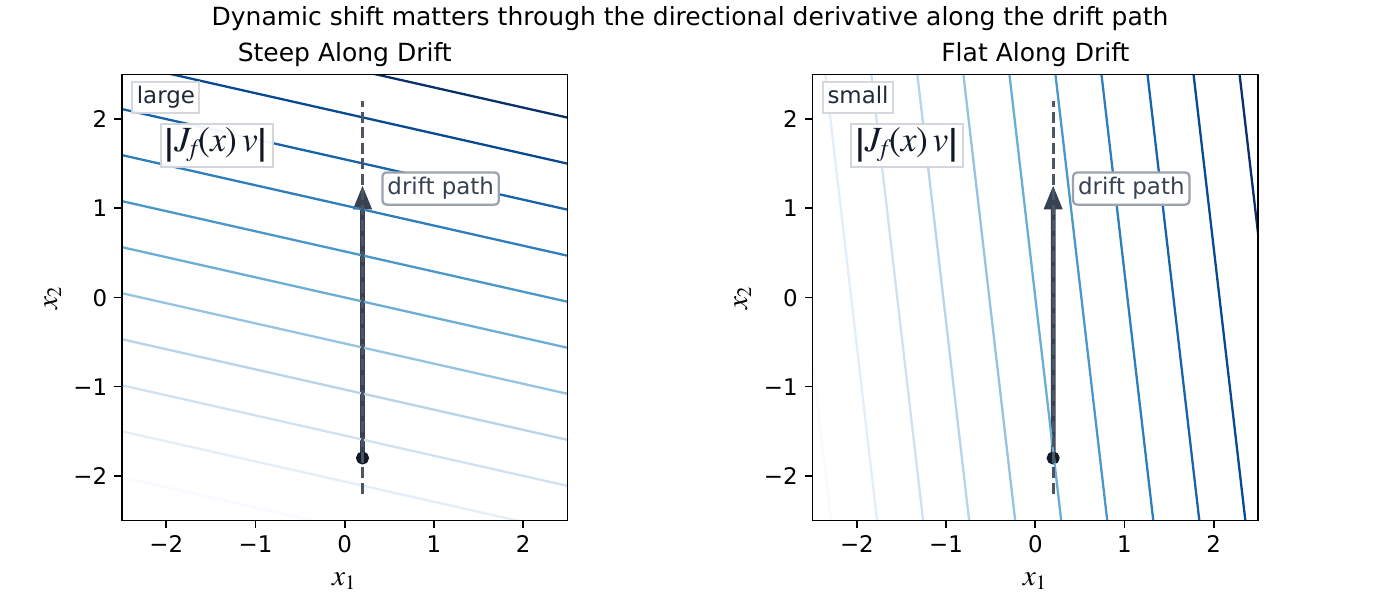}
  \caption{Geometric intuition for drift-aligned instability. The deployment path is the same in both panels, but the local tangent geometry is different. When the drift direction aligns with a steep tangent direction, the score changes rapidly, whereas when the model is flat along that same direction, the score remains stable. Here $v$ denotes the local drift direction. This figure visualizes the directional quantity that appears in Theorem~\ref{thm:jv} and \eqref{eq:jv-bound}.}
  \label{fig:geometry}
\end{figure*}

\section{Setup and Assumptions}
We use $\|\cdot\|$ for Euclidean vector norms and $\|\cdot\|_F$ for Frobenius matrix norms.
Let $t\in[0,T]$ index deployment time. Let $X_t\in\R^d$ denote the covariate process, and let $f_\theta:\R^d\to\R$ be a frozen scalar score with row Jacobian $J_f(x):=\nabla f_\theta(x)^\top$. We study a scalar deployment performance field $g_\theta:\R^d\to\R$ and the associated trajectory
\begin{equation}
\label{eq:risk-trajectory}
r(t):=\E[g_\theta(X_t)].
\end{equation}
The scalar presentation keeps the theorem notation minimal. For a vector score $f_\theta:\R^d\to\R^m$, $J_f(x)\in\R^{m\times d}$ and the same results hold under the unchanged directional condition $|\nabla g_\theta(X_t)^\top\dot X_t|\le\beta\|J_f(X_t)\dot X_t\|$. In particular, if $g_\theta=h\circ f_\theta$ and $\|\nabla h\|\le\beta$, this condition follows from the chain rule and Cauchy--Schwarz. The multiclass experiments apply the matrix formulation to centered logits, removing the common logit direction to which softmax probabilities are invariant.
For classification, let $\ell$ be the loss and $P_0(\cdot\mid x)$ the time-invariant label conditional. A natural example is
\[
g_\theta(x)=\E_{Y\sim P_0(\cdot\mid x)}[\ell(f_\theta(x),Y)],
\]
but the theorem below does not rely on this representation alone. Instead, we state the regularity needed to make the risk-trajectory argument defensible.
The point of this setup is to separate two roles that are often conflated in shift analyses. The model is frozen, while the covariate stream is allowed to move.

To measure temporal instability, draw $U\sim\mathrm{Unif}[0,T]$ and define
\begin{equation}
\label{eq:temporal-volatility}
\Var_U(r(U))=\E_U[(r(U)-\E_U r(U))^2].
\end{equation}
The quantity in \eqref{eq:temporal-volatility} is the notion of deployment volatility studied throughout the paper.

\textit{Assumption A1 (trajectory regularity).} The sample paths $t\mapsto X_t$ are almost surely absolutely continuous and
\[
\E\int_0^T \|\dot X_t\|^2dt < \infty.
\]

\textit{Assumption A2 (chain rule under the expectation).} Assume $\E|g_\theta(X_0)|<\infty$. The field $g_\theta$ is weakly differentiable, $t\mapsto g_\theta(X_t)$ is almost surely absolutely continuous, and
\[
g_\theta(X_t)-g_\theta(X_s)=\int_s^t \nabla g_\theta(X_u)^\top \dot X_u\,du,
\]
for all $0\le s<t\le T$ on almost every sample path, with
\[
\E\int_0^T |\nabla g_\theta(X_t)^\top \dot X_t|dt < \infty.
\]

\textit{Assumption A3 (score Jacobian well posed and directional domination along deployment).} There exists a measurable matrix field $J_f$ such that $J_f(X_t)$ is well defined almost surely for almost every $t\in[0,T]$,
\[
\E\int_0^T \|J_f(X_t)\dot X_t\|^2dt < \infty,
\]
and there exists $\beta>0$ such that
\[
|\nabla g_\theta(X_t)^\top \dot X_t| \le \beta \|J_f(X_t)\dot X_t\|
\]
for almost every $t\in[0,T]$ on almost every sample path.

Assumption A3 is the paper's main compatibility condition. It should not be read as a consequence of dynamic covariate shift alone. Rather, it asserts that along the realized deployment path, the local change in the chosen performance field is dominated by the local change in the score, and that the score Jacobian is well posed there.

\begin{remark}[What A3 means, when it holds, and when it fails]
\label{rem:a3}
A3 is an along-path condition rather than a pointwise condition in every ambient direction. The proof only needs control of $\nabla g_\theta(X_t)^\top \dot X_t$ by $J_f(X_t)\dot X_t$ along the realized velocity.

A3 holds in the natural composition case. If $g_\theta(x)=h(f_\theta(x))$, then
\[
\nabla g_\theta(X_t)^\top \dot X_t
=
h'(f_\theta(X_t))\,J_f(X_t)\dot X_t,
\]
so A3 follows with $\beta=\sup|h'|$ whenever $J_f(X_t)$ is well defined almost surely for almost every $t$. For Bernoulli cross-entropy against a fixed soft target, $|h'|\le 1$, giving $\beta=1$. The theorem is therefore not specific to rectified linear unit (ReLU) networks. Smooth predictors fit directly, and piecewise-linear predictors such as ReLU networks fit as soon as the deployment path avoids their nondifferentiability set almost surely for almost every deployment time.

A3 does not follow from pure covariate shift alone. The general risk $g_\theta(x)=\E_{Y\sim P_0(\cdot\mid x)}[\ell(f_\theta(x),Y)]$ carries an extra term from the spatial variation of the label conditional, not necessarily dominated by the score-Jacobian term even when that conditional is fixed over deployment time. Proposition~\ref{prop:residual} below retains this contribution explicitly instead of absorbing it into A3. Simultaneous concept shift remains outside the present setup because it makes the performance field itself time dependent. Exact domination can also fail at score saturation wherever $\nabla f_\theta(x)=0$ but $\nabla g_\theta(x)\neq 0$.
\end{remark}

\section{Main Theorem Package}
The first step is purely temporal.

\begin{lemma}[Time-domain derivative-energy control]
\label{lem:poincare}
If $r$ is absolutely continuous on $[0,T]$, then
\begin{equation}
\label{eq:poincare-bound}
\Var_U(r(U)) \le \frac{T}{\pi^2}\int_0^T (r'(t))^2dt.
\end{equation}
\end{lemma}

\begin{proof}
Let $\bar r=T^{-1}\int_0^T r(t)\,dt$. The Wirtinger inequality yields
\[
\int_0^T (r(t)-\bar r)^2dt \le \frac{T^2}{\pi^2}\int_0^T (r'(t))^2dt.
\]
Divide by $T$.
\end{proof}

Lemma~\ref{lem:poincare} says volatility is impossible without derivative energy (see \eqref{eq:poincare-bound}). The next result identifies the geometric driver of that energy.

\begin{theorem}[Jacobian-velocity control of risk volatility]
\label{thm:jv}
Assume A1--A3. Then $r$ is absolutely continuous and
\begin{equation}
\label{eq:jv-bound}
\Var_U(r(U)) \le \frac{\beta^2 T}{\pi^2}\int_0^T \E\|J_f(X_t)\dot X_t\|^2dt.
\end{equation}
\end{theorem}

\begin{proof}
Fix $0\le s<t\le T$. By Assumption A2 and Fubini,
\begin{align*}
r(t)-r(s)
&=
\E\!\left[\int_s^t \nabla g_\theta(X_u)^\top \dot X_u\,du\right] \\
&=
\int_s^t \E[\nabla g_\theta(X_u)^\top \dot X_u]\,du.
\end{align*}
Hence $r$ is absolutely continuous and, for almost every $t$,
\[
r'(t)=\E[\nabla g_\theta(X_t)^\top \dot X_t].
\]
Jensen's inequality and Assumption A3 then give
\[
\begin{aligned}
(r'(t))^2
&\le \E[(\nabla g_\theta(X_t)^\top \dot X_t)^2] \\
&\le \beta^2 \E\|J_f(X_t)\dot X_t\|^2.
\end{aligned}
\]
Apply Lemma~\ref{lem:poincare} and \eqref{eq:poincare-bound}.
\end{proof}

Theorem~\ref{thm:jv} isolates a single instability mechanism for the trajectory in \eqref{eq:risk-trajectory}. Deployment volatility is driven by accumulated tangent amplification along the realized motion of the data stream, as captured by \eqref{eq:jv-bound}. Figure~\ref{fig:geometry} is a geometric rendering of that statement. In that sense, the theorem turns Jacobian-based sensitivity analysis and regularization \cite{sokolic-jacobian,novak-sensitivity,kim-jacobian} into a statement about temporal deployment risk rather than static local robustness. The bound also has a useful operational reading. Volatility requires two ingredients at once, data motion and tangent sensitivity. If either one is small, long-horizon fluctuation remains small.

The exact form of A3 is useful because it yields the clean bound in \eqref{eq:jv-bound}, but it is stronger than covariate shift alone. The next result makes the gap explicit and separates the part that tangent regularization can control from the part it cannot.

\begin{proposition}[Jacobian control with a conditional-risk remainder]
\label{prop:residual}
Assume A1--A2 and the well-posedness and square-integrability conditions on $J_f$ in A3. Suppose there is a nonnegative measurable process $q_t$ with
\[
\E\int_0^T q_t^2dt<\infty
\]
such that, almost surely for almost every $t$,
\begin{equation}
\label{eq:residual-domination}
|\nabla g_\theta(X_t)^\top\dot X_t|
\le
\beta\|J_f(X_t)\dot X_t\|+q_t.
\end{equation}
Define $\gamma_t:=\|J_f(X_t)\dot X_t\|$. Then $r$ is absolutely continuous and
\begin{align}
\Var_U(r(U))
&\le
\frac{T}{\pi^2}\int_0^T
\E[(\beta\gamma_t+q_t)^2]dt
\label{eq:residual-bound}\\
&\le
\frac{2\beta^2T}{\pi^2}\int_0^T\E\gamma_t^2dt
\nonumber\\
&\quad+
\frac{2T}{\pi^2}\int_0^T\E q_t^2dt.
\label{eq:residual-bound-separated}
\end{align}
\end{proposition}

\begin{proof}
Assumption A2 gives the same expression for $r'(t)$ as in Theorem~\ref{thm:jv}. Jensen's inequality and \eqref{eq:residual-domination} give
\[
(r'(t))^2
\le
\E[(\beta\gamma_t+q_t)^2].
\]
Lemma~\ref{lem:poincare} yields \eqref{eq:residual-bound}; the inequality $(a+b)^2\le 2a^2+2b^2$ yields \eqref{eq:residual-bound-separated}.
\end{proof}

This remainder has a direct interpretation for conditional risk. Suppose $P_0(dy\mid x)=p(y\mid x)\nu(dy)$ for a reference measure $\nu$, and assume the derivatives below may be passed through the integral. Here $\partial_1\ell$ denotes differentiation of the loss with respect to its first argument, the model score. For
\[
g_\theta(x)=\int \ell(f_\theta(x),y)p(y\mid x)\nu(dy),
\]
the directional derivative along any $v\in\R^d$ decomposes as
\begin{align}
\nabla g_\theta(x)^\top v
&=
\left(\int \partial_1\ell(f_\theta(x),y)p(y\mid x)\nu(dy)\right)J_f(x)v
\nonumber\\
&\quad+
\int \ell(f_\theta(x),y)\nabla_xp(y\mid x)^\top v\,\nu(dy).
\label{eq:conditional-risk-decomposition}
\end{align}
If the conditional mean loss slope satisfies
\[
\left|
\int \partial_1\ell(f_\theta(X_t),y)
p(y\mid X_t)\nu(dy)
\right|\le L
\]
almost surely along deployment, Proposition~\ref{prop:residual} applies with $\beta=L$ and
\begin{equation}
\label{eq:conditional-remainder}
q_t=
\left|
\int \ell(f_\theta(X_t),y)
\nabla_xp(y\mid X_t)^\top\dot X_t\,\nu(dy)
\right|,
\end{equation}
provided this remainder is square integrable. A pointwise bound $|\partial_1\ell|\le L$ is sufficient but not necessary for the displayed slope condition. The first term in \eqref{eq:conditional-risk-decomposition} is model mediated: it is the change in loss caused by movement of the frozen score. The second is induced by spatial variation of the fixed label conditional along the same covariate path. DTR can reduce the former through $J_f$. It does not in general control the latter. Theorem~\ref{thm:jv} is recovered when the remainder vanishes or is itself dominated by the score-Jacobian term. This distinction sharpens the scope of the geometric claim without requiring the conditional risk to be a pure composition $h\circ f_\theta$.

The low-rank regime gives the algorithmic specialization.

\begin{corollary}[Drift-subspace bound]
\label{cor:lowrank}
Assume the hypotheses of Theorem~\ref{thm:jv}. If
\[
\dot X_t = Va_t + \rho_t,
\qquad
V\in\R^{d\times k},
\]
with $a_t\in\R^k$, $\rho_t\in\R^d$, $V^\top V=I_k$, and $V^\top \rho_t=0$, then
\begin{equation}
\label{eq:lowrank-bound}
\Var_U(r(U))
\le
\frac{2\beta^2 T}{\pi^2}(B_V+B_\rho),
\end{equation}
where
\begin{align*}
B_V &:= \int_0^T \E[\|J_f(X_t)V\|_F^2\|a_t\|^2]dt, \\
B_\rho &:= \int_0^T \E\|J_f(X_t)\rho_t\|^2dt.
\end{align*}
\end{corollary}

\begin{proof}
By Theorem~\ref{thm:jv} and \eqref{eq:jv-bound},
\[
\begin{aligned}
\|J_f(X_t)\dot X_t\|^2
&=
\|J_f(X_t)Va_t+J_f(X_t)\rho_t\|^2 \\
&\le
2\|J_f(X_t)Va_t\|^2+2\|J_f(X_t)\rho_t\|^2.
\end{aligned}
\]
Since $V$ has orthonormal columns,
\[
\|J_f(X_t)Va_t\|^2 \le \|J_f(X_t)V\|_F^2\|a_t\|^2.
\]
Integrate over $t$.
\end{proof}

Corollary~\ref{cor:lowrank} is the paper's main reduction. When $B_\rho$ is small, the leading stability term in \eqref{eq:lowrank-bound} is governed by the directional Jacobian energy in the drift subspace rather than by isotropic smoothness. This is precisely the regime in which structured temporal drift can be a useful approximation \cite{yao-wildtime} and low-rank change models become algorithmically plausible \cite{xie-lowrank-allerton2016}.

The same conclusion survives the remainder extension, with its scope made explicit. Write $Q:=\int_0^T\E q_t^2dt$. Combining Proposition~\ref{prop:residual} with the low-rank decomposition gives
\begin{equation}
\label{eq:residual-lowrank-bound}
\begin{split}
\Var_U(r(U))
&\le
\frac{4\beta^2T}{\pi^2}(B_V+B_\rho)\\
&\quad+
\frac{2T}{\pi^2}Q.
\end{split}
\end{equation}

\begin{corollary}[Anisotropic tangent-energy control]
\label{cor:anisotropic}
Under Corollary~\ref{cor:lowrank}, suppose that for almost every $t$,
almost surely,
\[
\|a_t\|^2\le A,
\qquad
\|\rho_t\|^2\le R.
\]
Let $P_V^\perp=I-VV^\top$ and define
\[
\begin{aligned}
\mathcal E_\parallel
&:=\int_0^T\E\|J_f(X_t)V\|_F^2dt,\\
\mathcal E_\perp
&:=\int_0^T\E\|J_f(X_t)P_V^\perp\|_F^2dt.
\end{aligned}
\]
Then
\begin{equation}
\label{eq:anisotropic-bound}
\Var_U(r(U))
\le
\frac{2\beta^2T}{\pi^2}
\left(A\mathcal E_\parallel+R\mathcal E_\perp\right).
\end{equation}
Under Proposition~\ref{prop:residual}, the corresponding bound is
\begin{equation}
\label{eq:anisotropic-residual-bound}
\Var_U(r(U))
\le
\frac{4\beta^2T}{\pi^2}
\left(A\mathcal E_\parallel+R\mathcal E_\perp\right)
+\frac{2T}{\pi^2}Q.
\end{equation}
\end{corollary}

\begin{proof}
The definition of $B_V$ and the bound on $a_t$ give
$B_V\le A\mathcal E_\parallel$. Since $V^\top\rho_t=0$,
$\rho_t=P_V^\perp\rho_t$, and therefore
\[
\|J_f(X_t)\rho_t\|^2
\le
\|J_f(X_t)P_V^\perp\|_F^2\|\rho_t\|^2.
\]
Hence $B_\rho\le R\mathcal E_\perp$. Substitution into
\eqref{eq:lowrank-bound} and \eqref{eq:residual-lowrank-bound} gives the two claims.
\end{proof}

Corollary~\ref{cor:anisotropic} is the direct theorem-to-objective bridge. DTR acts on $\mathcal E_\parallel$, while an orthogonal penalty acts on $\mathcal E_\perp$. Their natural relative scale reflects the anticipated aligned and residual velocity budgets $A$ and $R$. Pure DTR is the efficient choice when $R$ is negligible. When the estimated subspace is incomplete, the bound prescribes a smaller orthogonal guardrail. The term $Q$ remains outside Jacobian control and instead calls for labels, adaptation, or a richer predictor.

\section{Method}
The theorem package suggests a simple design principle for the model-mediated term. If instability is created by motion through steep directions, then regularization should flatten the model along directions the deployment path is expected to traverse, not everywhere at once.

\subsection{Drift-aligned tangent regularization}
Let $V$ estimate the dominant drift subspace from historical or predeployment covariates, with labeled calibration exposures used when available. Retrospective deployment estimates are used only in explicitly designated analyses. Two minimal estimators are sufficient for the present paper. For window mean $\mu_t$ and lag $\Delta$, one is a mean-difference direction
\[
\Delta\mu_t=\mu_t-\mu_{t-\Delta},
\qquad
v_t=\frac{\Delta\mu_t}{\|\Delta\mu_t\|},
\]
or the top-$k$ principal directions of the difference cloud $\{X_t-X_{t-\Delta}\}$ over a rolling window.

With a fixed subspace estimate $V$, we train using
\begin{equation}
\label{eq:dtr-objective}
\mathcal L_{\mathrm{DTR}}(\theta)
=
\E_{(X,Y)}\ell(f_\theta(X),Y)
\;+\;
\lambda\E_X\|J_f(X)V\|_F^2.
\end{equation}
The regularization weight $\lambda\geq0$ controls the directional penalty. The objective in \eqref{eq:dtr-objective} is not global Jacobian smoothing. It is a directional sensitivity constraint. The model is flattened only along directions expected to dominate future drift. Corollary~\ref{cor:lowrank} explains why this is the quantity worth shrinking, and it distinguishes DTR from isotropic Jacobian penalties that aim at broader smoothness or margin control \cite{sokolic-jacobian,kim-jacobian}. It is also conceptually adjacent to empirical sensitivity analyses that track related Jacobian quantities without imposing directional structure \cite{novak-sensitivity}.

Corollary~\ref{cor:anisotropic} also prescribes a guarded extension. For an orthonormal $V$, let $P_V=VV^\top$ and $P_V^\perp=I-P_V$. We define anisotropic DTR by
\begin{equation}
\label{eq:anisotropic-dtr}
\begin{split}
\mathcal L_{\mathrm{A\text{-}DTR}}(\theta)
&=\E\ell(f_\theta(X),Y)
+\lambda_\parallel\E\|J_f(X)V\|_F^2\\
&\quad+\lambda_\perp\E\|J_f(X)P_V^\perp\|_F^2,
\qquad 0\leq\lambda_\perp\leq\lambda_\parallel.
\end{split}
\end{equation}
This single family contains standard training $(\lambda_\perp=\lambda_\parallel=0)$, pure DTR $(\lambda_\perp=0)$, and isotropic Jacobian smoothing $(\lambda_\perp=\lambda_\parallel)$. For implementation, \eqref{eq:anisotropic-dtr} is equivalently an isotropic penalty of weight $\lambda_\perp$ plus an additional aligned penalty of weight $\lambda_\parallel-\lambda_\perp$. The orthogonal term is therefore a guardrail, not a retreat from directional structure: it controls residual model sensitivity while preserving a larger regularization budget along expected drift.
The corollary integrates tangent energy along deployment, whereas \eqref{eq:anisotropic-dtr} estimates it on training or calibration inputs. The empirical objective is therefore a predeployment surrogate: it requires those inputs to represent the local model geometry likely to be encountered along future motion.

\subsection{Monitoring score}
The same geometry yields a theorem-motivated proxy diagnostic. With current drift-subspace estimate $V_t$, define speed $s_t$, energy $G_t$, and product $h_t$ by
\begin{equation}
\label{eq:hazard-score}
\begin{aligned}
s_t&:=\|\Delta\mu_t\|/\Delta,
&
G_t&:=\E\|J_f(X_t)V_t\|_F^2, \\
h_t&:=s_t^2G_t.
\end{aligned}
\end{equation}
In the experiments we also report short rolling averages of $h_t$, which better match the accumulated-energy form of the volatility bound.

\begin{proposition}[Rank-$1$ bookkeeping for the hazard score]
\label{prop:hazard}
Consider the rank-$1$ monitoring proxy used below, so that $V_t=v_t\in\R^d$ is a unit vector. Suppose the block mean difference satisfies
\[
\frac{\Delta\mu_t}{\Delta}=v\bar a_t+\bar\rho_t,
\qquad
\|v\|=1,
\qquad
v^\top \bar\rho_t=0,
\]
where $\bar a_t$ is the block-averaged aligned drift coefficient and $\bar\rho_t$ is the block-averaged residual drift. Write
\[
v_t=\cos\theta_t\,v+\sin\theta_t\,u_t,
\qquad
u_t^\top v=0,
\qquad
\|u_t\|=1,
\]
and define
\[
\begin{aligned}
G_{\parallel,t}&:=\E\|J_f(X_t)v\|^2, \\
G_{\perp,t}&:=\E\|J_f(X_t)u_t\|^2, \\
C_t&:=\E\langle J_f(X_t)v,J_f(X_t)u_t\rangle.
\end{aligned}
\]
Then
\[
s_t^2=|\bar a_t|^2+\|\bar\rho_t\|^2
\]
and
\[
G_t=\cos^2\theta_t\,G_{\parallel,t}
\;+\;
\sin^2\theta_t\,G_{\perp,t}
\;+\;
2\sin\theta_t\cos\theta_t\,C_t.
\]
Consequently,
\[
\begin{aligned}
h_t&=(|\bar a_t|^2+\|\bar\rho_t\|^2) \\
&\qquad\times
\bigl[
\cos^2\theta_t\,G_{\parallel,t}
\;+\;
\sin^2\theta_t\,G_{\perp,t}
\\
&\qquad\quad+
2\sin\theta_t\cos\theta_t\,C_t
\bigr].
\end{aligned}
\]
\end{proposition}

\begin{proof}
The identities are immediate from orthogonal decompositions. The first uses $(\Delta\mu_t)/\Delta=v\bar a_t+\bar\rho_t$ with $v^\top\bar\rho_t=0$. The second expands $J_f(X_t)v_t=\cos\theta_t\,J_f(X_t)v+\sin\theta_t\,J_f(X_t)u_t$ and takes expectations.
\end{proof}

Proposition~\ref{prop:hazard} makes the proxy gap explicit in the rank-$1$ monitoring setting for the score in \eqref{eq:hazard-score}. The finite-window drift statistic contributes the residual term $\|\bar\rho_t\|^2$, while angular error mixes aligned, orthogonal, and overlap energies through $\theta_t$, $G_{\perp,t}$, and $C_t$. In that precise bookkeeping sense, $h_t$ matches the leading low-rank quantity from Corollary~\ref{cor:lowrank} up to block averaging, residual drift, and subspace-estimation error rather than reproducing the corollary integrand exactly. It becomes large only when the stream is moving quickly \emph{and} the predictor is steep along that motion. Relative to dataset-shift diagnostics \cite{rabanser-shift}, uncertainty diagnostics under shift \cite{ovadia-shift}, and sequential detection procedures \cite{wu-rscusum,cao-xie-allerton2015,cooper-allerton2024}, the point of $h_t$ is alignment rather than optimality. It monitors the same directional mechanism singled out by the theory without claiming to be a sequentially optimal detector.
Training and monitoring therefore share the same geometry: the directions used to suppress future instability are reused to measure when that instability is likely.

\section{Experiments}
\textit{Evaluation protocol.}
Hyperparameters use only training/validation windows. Deployment is held out for final evaluation. Unless a second policy is stated explicitly, penalties are selected by validation loss, with validation directional energy only a secondary tie-breaker. Metrics are post-selection matched-seed means: volatility is the discrete blockwise analogue of \eqref{eq:temporal-volatility}, derivative energy is the empirical integral in \eqref{eq:poincare-bound}, directional energy is drift-direction Jacobian energy, and terminal risk is final-block loss.

\subsection{Synthetic time-domain sanity check}
The synthetic study uses the smallest setup that exposes the paper's mechanism. There is one stable signal coordinate, one drifting nuisance coordinate, and a frozen classifier whose tangent geometry can either amplify or suppress that drift. Motivated by temporal-shift evaluation settings \cite{yao-wildtime,han-temporal}, the drift direction is known by construction so the theorem-to-method link can be read cleanly. Labels are sampled as $Y\sim\mathrm{Bernoulli}(1/2)$ and we write $S=2Y-1\in\{-1,+1\}$. Conditional on $S$, features are generated as
\[
x_1\sim\mathcal N(1.05S,0.90^2),
\qquad
x_2\sim\mathcal N(1.25S+\delta(t),0.55^2).
\]
The first coordinate is a stable signal. The second is a nuisance coordinate that is predictive at training time but drifts during deployment. Only $\delta(t)$ changes, preserving labels along transported sample paths while generally changing $P_t(Y\mid X)$. Thus this benchmark tests the directional mechanism under label-preserving transport rather than pure covariate shift.
The benchmark is intentionally simple. Its job is to make the paper's directional mechanism observable with as little confounding structure as possible.

The drift schedule is monotone and smooth, with speed
\begin{align*}
\dot\delta(t)
&=
0.55
+ 1.05e^{-((t-0.33)/0.10)^2}
+ 0.75e^{-((t-0.76)/0.08)^2}, \\
\delta(0) &= 0,
\end{align*}
so the stream contains two periods of rapid motion. We train a small ReLU classifier on samples drawn at $t=0$ and evaluate it on fresh samples along the deployment path. The sweep uses $\lambda\in\{0,0.01,0.03,0.08\}$ over $20$ random seeds. DTR uses the true drift direction $V=e_2$ so that this experiment isolates the theorem-to-method link rather than subspace-estimation error.

Figure~\ref{fig:theory} illustrates the temporal step in the paper's argument. Each point is one trained model, plotted by empirical risk volatility in \eqref{eq:temporal-volatility} versus the derivative-energy quantity from Lemma~\ref{lem:poincare} and \eqref{eq:poincare-bound}. The points sit below the diagonal, as they should, and DTR moves the cloud toward the lower-left corner. The Jacobian-velocity expression also upper-bounds empirical volatility in this transport example, but it is numerically looser, so the figure emphasizes the sharper temporal step. Averaging over $20$ seeds, mean volatility falls from $3.25\times10^{-3}$ for standard training to $2.39\times10^{-4}$ across the DTR sweep, while mean directional energy falls from $41.5$ to $1.85$.

\begin{figure}[t]
  \centering
  \includegraphics[width=\columnwidth]{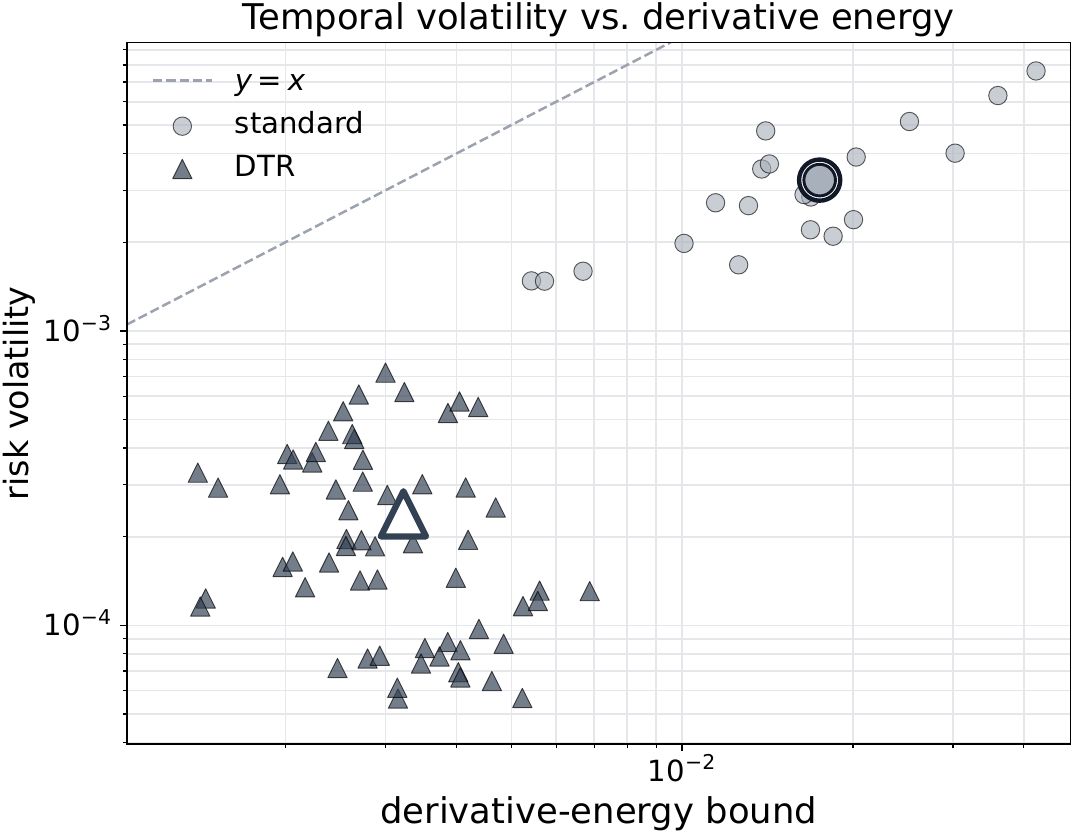}
  \caption{Synthetic time-domain sanity check on log--log axes. Each point is one trained model evaluated over the drifting deployment path. The horizontal axis is the derivative-energy quantity from Lemma~\ref{lem:poincare} and \eqref{eq:poincare-bound}, and the vertical axis is empirical volatility $\Var_U(r(U))$. Large outlined markers denote method means. DTR shifts models toward lower energy and lower volatility.}
  \label{fig:theory}
\end{figure}

\subsection{Directional versus isotropic smoothing}
The main algorithmic question is not whether Jacobian regularization helps in general, but whether directional smoothing beats isotropic smoothing when drift is concentrated in one known direction. We answer that question in the same synthetic environment by replacing the DTR penalty with an isotropic Jacobian penalty
\[
\lambda\E\|\nabla f_\theta(X)\|^2,
\]
while keeping the directional penalty
\[
\lambda\E\|J_f(X)V\|_F^2
\]
for DTR, with the same architecture and $20$ matched seeds over $\lambda\in\{0.01,0.03,0.08\}$. Under rank-$1$ drift, isotropic smoothing spends budget in directions the deployment path does not substantially traverse, whereas DTR spends it only along the realized motion. Table~\ref{tab:comparison} summarizes the mean values across the nonzero-$\lambda$ sweep. DTR improves on both baselines on every metric. Figure~\ref{fig:comparison} (left) shows the matched $\lambda=0.03$ comparison.

\begin{table}[t]
\caption{Directional comparison showing the mean over the nonzero-$\lambda$ sweep. Lower is better on all metrics.}
\centering
\setlength{\tabcolsep}{4pt}
\footnotesize
\begin{tabular}{lcccc}
\toprule
Method & Deriv.\ energy & Volatility & Dir.\ energy & Term.\ risk \\
\midrule
Standard   & $5.15\!\times\!10^{-3}$ & $3.25\!\times\!10^{-3}$ & $41.5$ & $0.189$ \\
Isotropic  & $5.85\!\times\!10^{-4}$ & $4.32\!\times\!10^{-4}$ & $2.17$ & $0.172$ \\
DTR        & $2.90\!\times\!10^{-4}$ & $2.39\!\times\!10^{-4}$ & $1.85$ & $0.136$ \\
\bottomrule
\end{tabular}
\label{tab:comparison}
\end{table}

We then test drift-subspace misspecification by rotating the one-dimensional subspace used by DTR. Writing the estimated direction as
\[
\hat v_\alpha=(\sin\alpha,\cos\alpha),
\]
its alignment with the true drift direction $e_2$ is $\cos\alpha$. With the correct subspace, the directional advantage remains strongest. With a mild $20^\circ$ rotation, volatility rises by a factor of $1.39$ and terminal risk by a factor of $1.18$ relative to aligned DTR, but both remain far below the standard model. The same cosine-sine bookkeeping used in Proposition~\ref{prop:hazard} helps explain that pattern: small angular error mainly rescales the aligned contribution, whereas large misalignment transfers weight to orthogonal energy. With an orthogonal subspace, the effect largely disappears. Derivative energy rises by a factor of $39.8$, volatility by $29.6$, and directional energy by $18.9$. Figure~\ref{fig:comparison} (right) therefore sharpens the paper's directional claim. Pure DTR is most efficient when the drift geometry is right. Its degradation under misspecification motivates the orthogonal guardrail tested in the gas-sensor study below.

\begin{figure*}[t]
  \centering
  \includegraphics[width=0.9\textwidth]{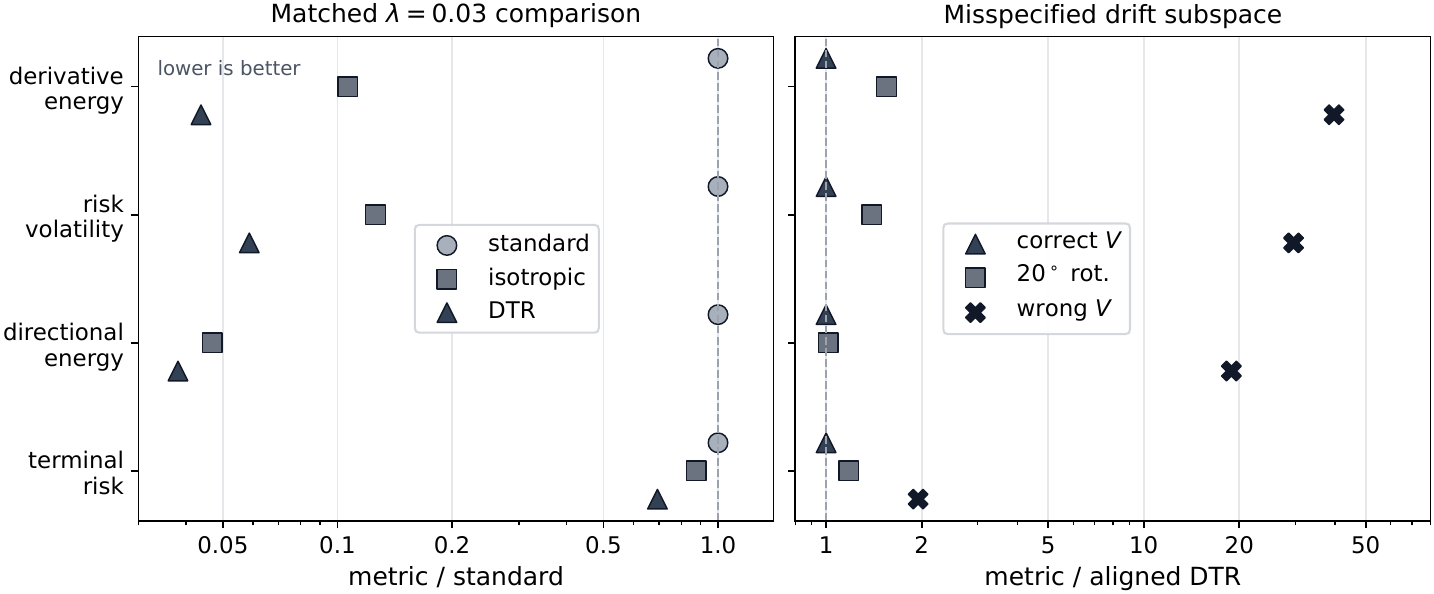}
  \caption{Synthetic directional comparisons. Left shows matched $\lambda=0.03$ ratios for derivative energy, risk volatility, directional energy, and terminal risk, normalized by the standard model. DTR is lower than isotropic Jacobian regularization on all four metrics. Right shows the same metrics normalized by aligned DTR for the correct subspace, a mild $20^\circ$ rotation, and a wrong subspace. Mild misspecification is tolerable, but wrong subspaces are not.}
  \label{fig:comparison}
\end{figure*}

\subsection{Conditional-risk remainder stress test}
We next stress-test Proposition~\ref{prop:residual} under pure covariate shift with a spatially varying but time-invariant label conditional. Let $X_t=2t$ for $t\in[0,1]$ and
\[
P_0(Y=1\mid X=x)=\eta_c(x):=\sigma(1+cx),
\]
where $\sigma(z)=(1+e^{-z})^{-1}$ is the logistic sigmoid. We fit a frozen scalar logistic score on source samples. For expected binary cross-entropy, the score-mediated coefficient has absolute value at most one, and the exact conditional-risk derivative has signed remainder $-\eta_c'(X_t)f_\theta(X_t)\dot X_t$. We therefore use $\beta=1$ and $q_t=|\eta_c'(X_t)f_\theta(X_t)\dot X_t|$. Figure~\ref{fig:conditional-remainder} summarizes $100$ seeds for each $c\in\{0,0.5,1,1.5,2\}$. At $c=0$, DTR reduces volatility from $9.20\times10^{-3}$ to $6.04\times10^{-8}$. At $c=2$, its directional energy remains only $1.20\times10^{-3}$ while volatility rises to $5.51\times10^{-3}$. Across all $400$ DTR runs with $c>0$, the Jacobian-only bound fails and the coupled remainder bound \eqref{eq:residual-bound} holds. Thus a small directional Jacobian is not sufficient when conditional-risk variation dominates, exactly as the proposition predicts.

\begin{figure*}[t]
  \centering
  \includegraphics[width=0.9\textwidth]{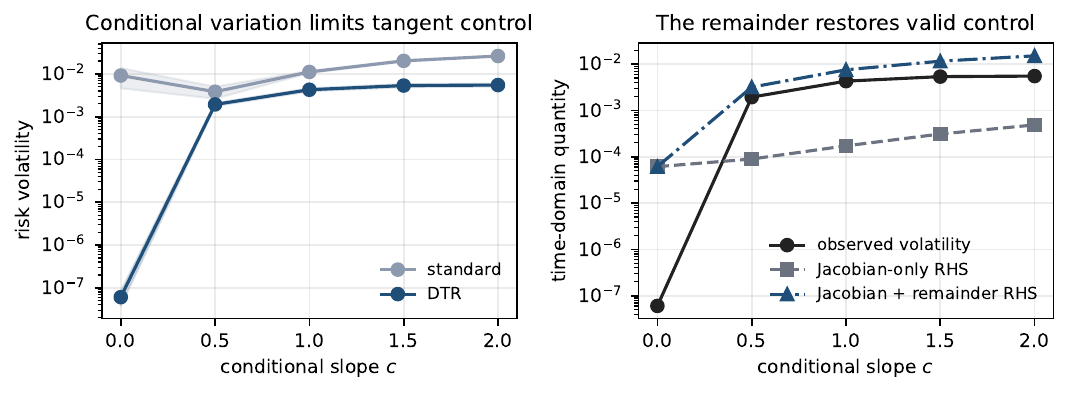}
  \caption{Conditional-risk remainder stress test. Left: DTR nearly removes volatility when the label conditional is spatially constant, but volatility returns as its slope $c$ grows. Right: for DTR, the Jacobian-only right-hand side fails for every run with $c>0$, while the coupled Jacobian-plus-remainder bound remains valid. Curves show means over $100$ seeds; shading shows the $95\%$ confidence interval (CI) for each mean.}
  \label{fig:conditional-remainder}
\end{figure*}

\subsection{Gas-sensor drift: directional control with an orthogonal guardrail}
We next test the method in the setting it is designed for: a frozen classifier exposed to long-horizon sensor drift. The UCI Gas Sensor Array Drift at Different Concentrations dataset contains measurements from $16$ metal-oxide sensors collected over three years in ten chronological batches \cite{uci-gas-drift,vergara-gas-drift,rodriguez-gas-calibration}. Eight response features per sensor give $128$ inputs. We classify the five gases represented throughout the split. Toluene is excluded because it is absent from several middle batches. Batches $1$--$2$ provide $1610$ training observations, batches $3$--$5$ provide $1944$ validation observations, and batches $6$--$10$ form a frozen deployment of $8523$ observations.

The drift basis uses only predeployment calibration data. To separate sensor ageing from the known exposure design, we regress each standardized feature on gas-specific quadratic functions of log concentration over batches $1$--$5$. We then take the singular vectors of the four consecutive residual batch-mean differences. Validation over ranks $1$--$4$ and the common penalty grid selects the full rank-$4$ span, still only four directions in the $128$-dimensional input. The classifier is a two-hidden-layer ReLU network trained with class-balanced cross-entropy and a centered-logit Jacobian penalty. We report class-balanced deployment cross-entropy (CE), its variance over deployment batches, final-batch CE, and macro accuracy. Macro accuracy is the unweighted mean of the five per-class accuracies. All choices use batches $1$--$5$ only. Deployment data are never used for model or hyperparameter selection and enter only final reporting over ten matched seeds.

Table~\ref{tab:gas-hybrid} compares the three endpoints of \eqref{eq:anisotropic-dtr} with two validation-defined interior configurations, which we call hybrids. ``CE rule'' minimizes mean validation CE. ``Stability rule'' minimizes validation volatility among configurations attaining at least $90\%$ validation macro accuracy. The CE-rule hybrid improves deployment CE, volatility, and final-batch CE over pure DTR in every matched seed. Its paired mean CE difference is $-0.257$ (bootstrap $95\%$ CI $[-0.329,-0.188]$), its volatility difference is $-0.161$ ($[-0.201,-0.126]$), and its macro-accuracy difference is $-0.020$ ($[-0.028,-0.012]$). The stability-rule hybrid occupies the complementary part of the frontier: relative to isotropic smoothing, it reduces CE by $0.299$ ($[-0.332,-0.262]$), improves macro accuracy by $0.098$ ($[0.080,0.117]$), and slightly increases volatility by $0.022$ ($[0.019,0.025]$). Thus pure DTR retains the highest average accuracy, isotropic smoothing gives the lowest volatility, and anisotropic DTR supplies the strongest joint risk--stability tradeoff.

\begin{table*}[t]
\caption{UCI gas-sensor frozen deployment. Hyperparameters are selected on batches $3$--$5$ and evaluated on batches $6$--$10$. Entries are mean $\pm$ standard deviation over $10$ matched seeds. Lower is better except for macro accuracy.}
\centering
\setlength{\tabcolsep}{5pt}
\footnotesize
\begin{tabular}{lccccc}
\toprule
Method $(\lambda_\perp,\lambda_\parallel)$ & Deploy CE & Volatility & Terminal CE & Macro acc. \\
\midrule
Standard $(0,0)$ & $2.151\pm0.229$ & $0.646\pm0.183$ & $3.010\pm0.363$ & $0.576\pm0.048$ \\
Pure DTR $(0,1)$ & $1.244\pm0.181$ & $0.340\pm0.096$ & $1.950\pm0.264$ & $\mathbf{0.680}\pm0.013$ \\
Isotropic $(1,1)$ & $1.267\pm0.071$ & $\mathbf{0.021}\pm0.004$ & $1.139\pm0.053$ & $0.503\pm0.041$ \\
Hybrid, CE rule $(0.003,3)$ & $0.987\pm0.115$ & $0.179\pm0.043$ & $1.493\pm0.166$ & $0.660\pm0.017$ \\
Hybrid, stability rule $(0.3,10)$ & $\mathbf{0.968}\pm0.056$ & $0.043\pm0.002$ & $\mathbf{1.085}\pm0.068$ & $0.600\pm0.038$ \\
\bottomrule
\end{tabular}
\label{tab:gas-hybrid}
\end{table*}

Two controls isolate the geometric claim. The rank-$4$ predeployment basis beats each of ten ambient-random rank-$4$ bases on deployment CE, volatility, terminal CE, and macro accuracy. The random-basis means are $2.117$, $0.621$, $2.960$, and $0.574$, respectively. Rank progression also matters: validation selection improves mean deployment CE from $2.158$ at rank $1$ to $1.452$, $1.330$, and $1.244$ at ranks $2$--$4$. These controls show that the gain is not generic low-rank smoothing. At class level, the CE-rule hybrid preserves DTR's strong performance on four gases but remains weak on Acetaldehyde. The stability-rule hybrid raises Acetaldehyde accuracy from $7.3\%$ to $37.2\%$ while distributing accuracy more evenly. Because gas concentration composition varies across deployment batches, we use class-balanced metrics and remove the predeployment gas--concentration response when estimating $V$. Residual concentration imbalance can nevertheless contribute to temporal risk.

\subsection{Complementary regression deployments}
We next ask whether the same directional pattern persists in regression settings where the drift subspace must be estimated from unlabeled covariates. We use Air Quality as a second sensor-drift case and Tetouan as a weather-driven load-prediction task with fixed label semantics and explicit exogenous covariates.
The regression results below use a retrospective subspace estimated from unlabeled deployment covariates. Section~\ref{sec:prospective} separately tests a strict predeployment estimate using only training and validation covariates.

The UCI Air Quality dataset records hourly outputs of a five-sensor gas array together with temperature and humidity in an Italian city. The associated field-study paper explicitly discusses sensor drift and seasonal influences over long deployments \cite{de-vito-airquality}. We predict the reference CO concentration $\mathrm{CO(GT)}$ from the five sensor channels $\mathrm{PT08.S1}$--$\mathrm{PT08.S5}$ together with temperature, relative humidity ($RH$), and absolute humidity ($AH$). Rows with the dataset's $-200$ missing-value marker are removed, leaving $1573$ training points from the first $12$ weeks, $580$ validation points from the next $4$ weeks, and $5191$ deployment points grouped into $20$ biweekly blocks. The predictor is frozen after training, and deployment risk is blockwise mean-squared error (MSE).
Because this dataset is explicitly a sensor-drift setting, the primary Air Quality DTR subspace is estimated from the sensor channels only. We first remove the supervised linear target direction learned on the training window. We then take the top two singular vectors of consecutive biweekly mean-shift vectors in that target-orthogonal sensor space and embed them back into the full feature space. This keeps weather covariates available to the predictor while avoiding a DTR penalty on the main target-predictive sensor direction. We compare standard training, isotropic Jacobian regularization, and DTR over $10$ matched random seeds. For reporting the deployment monitor, we use the blockwise rank-$1$ diagnostic: if $\mu_t$ is the standardized covariate mean in block $t$ and $\Delta=1$ block, then
\begin{align*}
s_t &= \|\mu_t-\mu_{t-1}\|, & v_t &= \frac{\mu_t-\mu_{t-1}}{\|\mu_t-\mu_{t-1}\|}, \\
G_t &= \E[(\nabla f_\theta(X_t)^\top v_t)^2], & h_t &= s_t^2 G_t.
\end{align*}
We also test the UCI Tetouan City power-consumption dataset \cite{uci-tetouan}. It is a ten-minute 2017 weather-and-load series with three zone-level targets. We predict Zone~1 power consumption from temperature, humidity, wind speed, diffuse-flow channels, and eight cyclic calendar features ($13$ inputs). Using the same frozen-model protocol, we train on January--April, validate on May--June, and deploy on July--December in six monthly blocks, yielding $17280/8784/26352$ training/validation/deployment points.

Table~\ref{tab:real-results} reports validation-MSE-selected results. Air Quality chooses $\lambda=0.08$ for isotropic smoothing and $\lambda=0.003$ for DTR, while Tetouan chooses $\lambda=3\times10^{-4}$ for isotropic smoothing and $\lambda=10^{-2}$ for DTR. On Air Quality, DTR improves mean deployment MSE, volatility, and terminal risk relative to standard training, while isotropic smoothing has the lowest mean MSE. On Tetouan, DTR improves both deployment MSE and volatility relative to both baselines (Figure~\ref{fig:tetouan}). The monitoring ablation in Table~\ref{tab:monitoring-volatility} complements these training results: rolling $h_t$ predicts future risk movement better than drift alone. Paired seed comparisons avoid per-block significance tests. Air Quality DTR improves deployment MSE and volatility in $9/10$ seeds against standard training; Tetouan DTR improves volatility in $8/10$ seeds against both standard and isotropic training, with bootstrap $95\%$ CIs for the paired mean differences below zero.

\begin{table*}[t]
\caption{Validation-selected real deployment results. Entries are mean $\pm$ standard deviation over $10$ matched random seeds. Lower is better for all metrics. For Tetouan, deployment MSE, volatility, and terminal risk are reported in units of $10^8$, $10^{16}$, and $10^7$, respectively.}
\centering
\setlength{\tabcolsep}{4pt}
\footnotesize
\begin{tabular}{llcccc}
\toprule
Dataset & Method & Deploy MSE & Volatility & Dir.\ energy & Terminal risk \\
\midrule
Air Quality & Standard & $0.449\pm0.069$ & $0.073\pm0.023$ & $0.079\pm0.008$ & $0.165\pm0.034$ \\
Air Quality & Iso.\ Jacobian & $0.415\pm0.030$ & $0.077\pm0.008$ & $0.053\pm0.005$ & $0.177\pm0.045$ \\
Air Quality & DTR & $0.432\pm0.058$ & $0.069\pm0.020$ & $0.079\pm0.008$ & $0.161\pm0.028$ \\
\addlinespace
Tetouan & Standard & $1.08\pm1.11$ & $1.07\pm1.66$ & $0.024\pm0.024$ & $1.57\pm0.41$ \\
Tetouan & Iso.\ Jacobian & $1.01\pm0.81$ & $0.720\pm1.024$ & $0.023\pm0.017$ & $1.61\pm0.41$ \\
Tetouan & DTR & $0.682\pm0.487$ & $0.307\pm0.432$ & $0.013\pm0.008$ & $1.45\pm0.45$ \\
\bottomrule
\end{tabular}
\label{tab:real-results}
\end{table*}

\begin{table*}[t]
\caption{Retrospective and strict predeployment rank-$2$ subspaces. Entries are means over $10$ matched seeds; Angle is the largest principal angle to the retrospective basis, Energy is captured deployment drift energy, and wins are paired DTR-vs-standard MSE/volatility wins. Tetouan deployment MSE and volatility are reported in units of $10^8$ and $10^{16}$, respectively.}
\centering
\setlength{\tabcolsep}{4pt}
\footnotesize
\begin{tabular}{llcccccc}
\toprule
Dataset & Basis & Angle & Energy & $\lambda$ & Deploy MSE & Volatility & Wins \\
\midrule
Air Quality & Standard & --- & --- & $0$ & $0.449\pm0.069$ & $0.073\pm0.023$ & --- \\
Air Quality & Retrospective & $0.0^\circ$ & $0.963$ & $0.003$ & $0.432\pm0.058$ & $0.069\pm0.020$ & $9/10,\;9/10$ \\
Air Quality & Predeployment & $59.2^\circ$ & $0.907$ & $0.003$ & $0.445\pm0.067$ & $0.074\pm0.023$ & $7/10,\;4/10$ \\
Air Quality & Random & $65.9^\circ$ & $0.657$ & $0.003$ & $0.460\pm0.072$ & $0.079\pm0.028$ & $2/10,\;2/10$ \\
\addlinespace
Tetouan & Standard & --- & --- & $0$ & $1.08\pm1.11$ & $1.07\pm1.66$ & --- \\
Tetouan & Retrospective & $0.0^\circ$ & $0.967$ & $0.010$ & $0.682\pm0.487$ & $0.307\pm0.432$ & $8/10,\;8/10$ \\
Tetouan & Predeployment & $16.3^\circ$ & $0.925$ & $0.010$ & $0.686\pm0.629$ & $0.360\pm0.702$ & $8/10,\;8/10$ \\
Tetouan & Random & $87.7^\circ$ & $0.177$ & $0.010$ & $0.678\pm0.686$ & $0.369\pm0.727$ & $9/10,\;9/10$ \\
\bottomrule
\end{tabular}
\label{tab:prospective}
\end{table*}

Table~\ref{tab:aq-subspace} shows why the Air Quality subspace definition matters. If $V$ is estimated from all covariates, validation-MSE selection chooses a large DTR penalty that suppresses block-direction energy but worsens deployment risk. The target-orthogonal sensor subspace avoids that failure: validation-MSE selection chooses $\lambda=0.003$, and DTR improves MSE and volatility in $9/10$ matched seeds against standard training. For the weather-residualized variant, we fit linear sensor responses to temperature and humidity using training and validation data, then estimate the drift basis from consecutive deployment block-mean differences of the sensor residuals. This weather-residualized sensor subspace gives an even lower mean MSE and volatility at a moderate $\lambda=0.03$. Validation MSE selects a larger penalty for that variant, so we report it only as a subspace sensitivity check. The practical lesson is the same as the misspecification experiment: DTR is useful when the estimated subspace represents nuisance drift rather than broad covariate motion or target signal.

\begin{table}[t]
\caption{Air Quality DTR subspace ablation. Entries are means over $10$ matched seeds. The first two DTR rows are validation-MSE selected; the weather-residualized row is a moderate-penalty sensitivity point. Wins are paired DTR-vs-standard MSE/volatility wins.}
\centering
\setlength{\tabcolsep}{3.5pt}
\footnotesize
\begin{tabular}{lccccc}
\toprule
Subspace & $\lambda$ & MSE & Vol. & Energy & Wins \\
\midrule
All covariates & $0.080$ & $0.485$ & $0.112$ & $0.044$ & $4/10,\;2/10$ \\
Target-orth.\ sensors & $0.003$ & $0.432$ & $0.069$ & $0.079$ & $9/10,\;9/10$ \\
Weather-resid.\ sensors & $0.030$ & $0.411$ & $0.061$ & $0.064$ & $9/10,\;6/10$ \\
\bottomrule
\end{tabular}
\label{tab:aq-subspace}
\end{table}

Finally, we evaluated the matched hazard score against block-to-block risk movement rather than raw risk level. This is the more natural target for the theory, since the bound controls temporal volatility. Table~\ref{tab:monitoring-volatility} reports Spearman correlations with next-block squared risk change. The one-block product $h_t$ is noisy, but a short rolling average of the matched score is positively associated with future risk movement on both real deployments. On Air Quality, the roll-2 score exceeds drift-only with a bootstrap $95\%$ confidence interval separated from the drift-only estimate. On Tetouan, the point estimate is also positive, but uncertainty is larger because there are only six deployment blocks per seed. Rolling $h_t$ reaches Spearman correlations of $0.31$--$0.35$ on both datasets. We treat it as a theorem-matched volatility diagnostic rather than a sequentially optimal detector.

\begin{table}[t]
\caption{Monitoring-score ablation on selected DTR deployments. Entries are Spearman correlations with next-block squared risk change $(r_{t+1}-r_t)^2$. Higher is better. Roll-2 and Roll-3 are trailing two- and three-block arithmetic averages of $h_t$, including the current block; only complete windows are evaluated.}
\centering
\setlength{\tabcolsep}{3pt}
\footnotesize
\begin{tabular}{lrrrrr}
\toprule
Dataset & Drift $s_t^2$ & Energy $G_t$ & Product $h_t$ & Roll-2 $h_t$ & Roll-3 $h_t$ \\
\midrule
Air Quality & $0.169$ & $-0.099$ & $0.044$ & $0.307$ & $0.353$ \\
Tetouan & $-0.339$ & $0.328$ & $0.218$ & $0.305$ & $0.335$ \\
\bottomrule
\end{tabular}
\label{tab:monitoring-volatility}
\end{table}

\begin{figure}[t]
  \centering
  \includegraphics[width=\columnwidth]{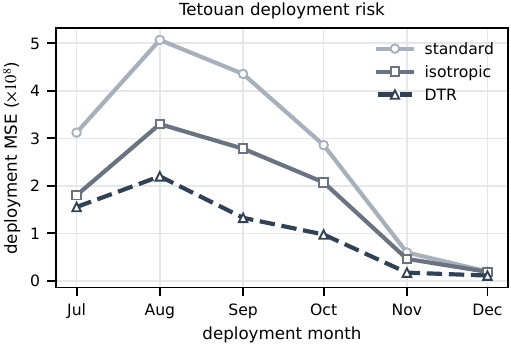}
  \caption{Tetouan frozen deployment over six monthly blocks. DTR stays below both standard and isotropic training throughout the deployment horizon, with the largest gap in the late-summer high-risk regime.}
  \label{fig:tetouan}
\end{figure}

\subsection{Prospective subspaces}
\label{sec:prospective}
The Air Quality and Tetouan results above use deployment covariates to estimate $V$. The gas-sensor experiment uses a strict predeployment basis. To remove look-ahead in the two regression studies, we evaluate the complete validation-selected DTR sweep with $V$ estimated only from block-mean motion in the training and validation windows: eight biweekly blocks for Air Quality and six monthly blocks for Tetouan. We retain the target-orthogonal sensor construction for Air Quality and use all standardized predictors for Tetouan. A fixed random rank-$2$ subspace provides a control. Table~\ref{tab:prospective} reports the largest principal angle to the retrospective basis and the fraction of deployment mean-difference energy captured by each basis.

The predeployment bases capture $90.7\%$ and $92.5\%$ of realized drift energy. On Air Quality, predeployment DTR beats standard MSE in $7/10$ seeds but volatility in only $4/10$. It nevertheless beats the random control in $9/10$ MSE and $7/10$ volatility comparisons. On Tetouan it nearly reproduces the retrospective result, beating standard on both metrics in $8/10$ seeds. The random control is also competitive on Tetouan, however, so that dataset supports prospective regularization but does not by itself isolate directional specificity.

\FloatBarrier
\section{Discussion and Limitations}
Covariate shift alone does not imply exact A3: even with time-invariant $P_0(Y\mid X)$, motion through $x$ can change conditional risk beyond the score Jacobian. Proposition~\ref{prop:residual} isolates this effect in $Q$, as the stress test confirms. DTR targets the Jacobian-mediated terms in \eqref{eq:residual-lowrank-bound}. A large remainder calls for a richer model, labels, or adaptation. Simultaneous concept shift remains outside the setup. ReLU models are admissible when deployment almost surely avoids kink sets for almost every time. Smooth models fit directly.

Two synthetic experiments test the temporal bound and low-rank implication. The remainder stress test probes Jacobian-only control. The gas-sensor study supplies the main field validation: its predeployment basis occupies four of $128$ directions, outperforms matched random subspaces, and shows that a small orthogonal guardrail can improve both risk and volatility over pure DTR. Air Quality and Tetouan extend the evidence to regression, emphasizing nuisance-aligned construction and prospective transfer.

The low-rank analysis depends on drift-subspace estimation. The gas study uses ten ambient-random controls. Air Quality and Tetouan use one fixed control each, so broader evidence remains valuable. The gas basis uses labeled exposures and known concentrations, which are natural in controlled calibration but not universal. Concentration composition changes across deployment batches despite class-balanced evaluation. Rolling $V_t$ suits monitoring. Training-time DTR requires a historical or predeployment basis unless the model is retrained. The Jacobian-velocity bound is actionable even when loose.

\balance
\section{Conclusion}
Long-horizon robustness under dynamic shift is a tangent-control problem: data velocity and local model geometry jointly govern model-mediated volatility. The remainder extension separates this actionable component from conditional-risk variation outside the score geometry. Under low-rank motion, the parallel--orthogonal decomposition gives a direct design rule: regularize strongly along expected drift and add only as much orthogonal smoothing as residual motion requires. Pure DTR is efficient when the subspace is accurate. Anisotropic DTR guards against incompleteness.

The experiments follow that arc. Synthetic studies isolate the inequality, alignment, and conditional-risk remainder. In the gas-sensor deployment, a strict predeployment rank-$4$ basis beats all ten random-subspace controls. The validation-selected hybrid improves deployment CE and volatility over pure DTR in every matched seed while preserving substantially more accuracy than isotropic smoothing. Air Quality and Tetouan provide complementary regression evidence. Together, the theoretical bounds, anisotropic DTR family, and controlled and field experiments support a common design rule: concentrate robustness along structured future motion, with a limited orthogonal guardrail when needed.

\vspace{-4pt}
\section*{Acknowledgment}
The author used AI tools, including Claude Code (Anthropic) and Codex (OpenAI), iteratively throughout this work. These tools assisted with manuscript development, editorial revision, experimental coding, figure preparation, and algebraic checks. Assistance extended across the manuscript and supporting computational materials. The author is responsible for all content and conclusions. The author reports no conflict of interest. No funding was received for this research.

\bibliographystyle{IEEEtran}
\bibliography{references}

\begin{thebibliography}{10}
\providecommand{\url}[1]{#1}
\csname url@samestyle\endcsname
\providecommand{\newblock}{\relax}
\providecommand{\bibinfo}[2]{#2}
\providecommand{\BIBentrySTDinterwordspacing}{\spaceskip=0pt\relax}
\providecommand{\BIBentryALTinterwordstretchfactor}{4}
\providecommand{\BIBentryALTinterwordspacing}{\spaceskip=\fontdimen2\font plus
\BIBentryALTinterwordstretchfactor\fontdimen3\font minus
  \fontdimen4\font\relax}
\providecommand{\BIBforeignlanguage}[2]{{%
\expandafter\ifx\csname l@#1\endcsname\relax
\typeout{** WARNING: IEEEtran.bst: No hyphenation pattern has been}%
\typeout{** loaded for the language `#1'. Using the pattern for}%
\typeout{** the default language instead.}%
\else
\language=\csname l@#1\endcsname
\fi
#2}}
\providecommand{\BIBdecl}{\relax}
\BIBdecl

\bibitem{ben-david-shift}
S.~Ben-David, J.~Blitzer, K.~Crammer, A.~Kulesza, F.~Pereira, and
  J.~Wortman~Vaughan, ``A theory of learning from different domains,''
  \emph{Machine Learning}, vol.~79, no. 1--2, pp. 151--175, 2010.

\bibitem{moreno-shift}
J.~G. Moreno-Torres, T.~Raeder, R.~Alaiz-Rodr{\'i}guez, N.~V. Chawla, and
  F.~Herrera, ``A unifying view on dataset shift in classification,''
  \emph{Pattern Recognition}, vol.~45, no.~1, pp. 521--530, 2012.

\bibitem{gama-drift}
J.~Gama, I.~{\v Z}liobait{\.e}, A.~Bifet, M.~Pechenizkiy, and A.~Bouchachia,
  ``A survey on concept drift adaptation,'' \emph{ACM Computing Surveys},
  vol.~46, no.~4, pp. 44:1--44:37, 2014.

\bibitem{koh-wilds}
P.~W. Koh, S.~Sagawa, H.~Marklund, S.~M. Xie, M.~Zhang, A.~Balsubramani, W.~Hu,
  M.~Yasunaga, R.~L. Phillips, I.~Gao, T.~Lee, E.~David, I.~Stavness, W.~Guo,
  B.~Earnshaw, I.~Haque, S.~M. Beery, J.~Leskovec, A.~Kundaje, E.~Pierson,
  S.~Levine, C.~Finn, and P.~Liang, ``Wilds: A benchmark of in-the-wild
  distribution shifts,'' in \emph{Proceedings of the 38th International
  Conference on Machine Learning}, ser. Proceedings of Machine Learning
  Research, vol. 139, 2021, pp. 5637--5664.

\bibitem{yao-wildtime}
H.~Yao, C.~Choi, B.~Cao, Y.~Lee, P.~W. Koh, and C.~Finn, ``Wild-time: A
  benchmark of in-the-wild distribution shift over time,'' in \emph{Advances in
  Neural Information Processing Systems}, vol.~35, 2022, pp. 10\,309--10\,324.

\bibitem{han-temporal}
E.~Han, C.~Huang, and K.~Wang, ``Model assessment and selection under temporal
  distribution shift,'' in \emph{Proceedings of the 41st International
  Conference on Machine Learning}, ser. Proceedings of Machine Learning
  Research, vol. 235, 2024, pp. 17\,374--17\,392.

\bibitem{landers-repo}
J.~R. Landers, ``jacobian-velocity-bounds,'' GitHub repository, 2026, [Online].
  Available: \url{https://github.com/jonland82/jacobian-velocity-bounds}.
  Accessed: May 15, 2026.

\bibitem{landers-site}
------, ``Jacobian-velocity bounds for deployment risk under covariate drift,''
  Project website, 2026, [Online]. Available:
  \url{https://jonland82.github.io/jacobian-velocity-bounds/}. Accessed: May
  15, 2026.

\bibitem{simchowitz-shift}
M.~Simchowitz, A.~Ajay, P.~Agrawal, and A.~Krishnamurthy, ``Statistical
  learning under heterogeneous distribution shift,'' in \emph{Proceedings of
  the 40th International Conference on Machine Learning}, ser. Proceedings of
  Machine Learning Research, vol. 202, 2023, pp. 31\,800--31\,851.

\bibitem{sun-ttt}
Y.~Sun, X.~Wang, Z.~Liu, J.~Miller, A.~A. Efros, and M.~Hardt, ``Test-time
  training with self-supervision for generalization under distribution
  shifts,'' in \emph{Proceedings of the 37th International Conference on
  Machine Learning}, ser. Proceedings of Machine Learning Research, vol. 119,
  2020, pp. 9229--9248.

\bibitem{wang-tent}
D.~Wang, E.~Shelhamer, S.~Liu, B.~Olshausen, and T.~Darrell, ``Tent: Fully
  test-time adaptation by entropy minimization,'' in \emph{International
  Conference on Learning Representations}, 2021.

\bibitem{gui-atta}
S.~Gui, X.~Li, and S.~Ji, ``Active test-time adaptation: Theoretical analyses
  and an algorithm,'' in \emph{International Conference on Learning
  Representations}, 2024.

\bibitem{sokolic-jacobian}
J.~Sokolic, R.~Giryes, G.~Sapiro, and M.~R.~D. Rodrigues, ``Robust large margin
  deep neural networks,'' \emph{IEEE Transactions on Signal Processing},
  vol.~65, no.~16, pp. 4265--4280, 2017.

\bibitem{kim-jacobian}
T.~Kim and H.~Yang, ``An infinite-width analysis on the jacobian-regularised
  training of a neural network,'' in \emph{Proceedings of the 41st
  International Conference on Machine Learning}, ser. Proceedings of Machine
  Learning Research, vol. 235, 2024, pp. 24\,584--24\,657.

\bibitem{novak-sensitivity}
R.~Novak, Y.~Bahri, D.~A. Abolafia, J.~Pennington, and J.~Sohl-Dickstein,
  ``Sensitivity and generalization in neural networks: an empirical study,'' in
  \emph{International Conference on Learning Representations}, 2018.

\bibitem{simard-tangentprop}
P.~Simard, B.~Victorri, Y.~LeCun, and J.~Denker, ``Tangent prop - a formalism
  for specifying selected invariances in an adaptive network,'' in
  \emph{Advances in Neural Information Processing Systems}, vol.~4, 1991, pp.
  895--903.

\bibitem{rifai-manifold-tangent}
S.~Rifai, Y.~N. Dauphin, P.~Vincent, Y.~Bengio, and X.~Muller, ``The manifold
  tangent classifier,'' in \emph{Advances in Neural Information Processing
  Systems}, vol.~24, 2011, pp. 2294--2302.

\bibitem{janvier-neural-de}
M.~Janvier, J.~Salomon, and E.~Meunier, ``Jacobian regularization stabilizes
  long-term integration of neural differential equations,'' arXiv preprint
  arXiv:2602.04608, 2026.

\bibitem{rabanser-shift}
S.~Rabanser, S.~G{\"u}nnemann, and Z.~C. Lipton, ``Failing loudly: An empirical
  study of methods for detecting dataset shift,'' in \emph{Advances in Neural
  Information Processing Systems}, vol.~32, 2019, pp. 1396--1408.

\bibitem{ovadia-shift}
Y.~Ovadia, E.~Fertig, J.~Ren, Z.~Nado, D.~Sculley, S.~Nowozin, J.~V. Dillon,
  B.~Lakshminarayanan, and J.~Snoek, ``Can you trust your model's uncertainty?
  evaluating predictive uncertainty under dataset shift,'' in \emph{Advances in
  Neural Information Processing Systems}, vol.~32, 2019, pp. 13\,991--14\,002.

\bibitem{wu-rscusum}
S.~Wu, E.~Diao, J.~Ding, T.~Banerjee, and V.~Tarokh, ``Robust quickest change
  detection for unnormalized models,'' in \emph{Proceedings of the Thirty-Ninth
  Conference on Uncertainty in Artificial Intelligence}, ser. Proceedings of
  Machine Learning Research, vol. 216, 2023, pp. 2314--2323.

\bibitem{cao-xie-allerton2015}
Y.~Cao and Y.~Xie, ``Multi-sensor gradual change detection,'' in
  \emph{Proceedings of the 53rd Annual Allerton Conference on Communication,
  Control, and Computing}, 2015, pp. 827--834.

\bibitem{xie-lowrank-allerton2016}
Y.~Xie and L.~M. Seversky, ``Sequential low-rank change detection,'' in
  \emph{Proceedings of the 54th Annual Allerton Conference on Communication,
  Control, and Computing}, 2016, pp. 128--133.

\bibitem{krishnamurthy-allerton2022}
V.~Krishnamurthy and L.~Snow, ``Quickest change detection using time
  inconsistent anticipatory and quantum decision modeling,'' in
  \emph{Proceedings of the 58th Annual Allerton Conference on Communication,
  Control, and Computing}, 2022, pp. 1--8.

\bibitem{cooper-allerton2024}
A.~Cooper and S.~P. Meyn, ``Quickest change detection using mismatched cusum
  extended abstract,'' in \emph{Proceedings of the 60th Annual Allerton
  Conference on Communication, Control, and Computing}, 2024, pp. 1--7.

\bibitem{uci-gas-drift}
A.~Vergara, ``Gas sensor array drift at different concentrations,'' UCI Machine
  Learning Repository, 2012, [Dataset].

\bibitem{vergara-gas-drift}
A.~Vergara, S.~Vembu, T.~Ayhan, M.~A. Ryan, M.~L. Homer, and R.~Huerta,
  ``Chemical gas sensor drift compensation using classifier ensembles,''
  \emph{Sensors and Actuators B: Chemical}, vol. 166--167, pp. 320--329, 2012.

\bibitem{rodriguez-gas-calibration}
I.~Rodriguez-Lujan, J.~Fonollosa, A.~Vergara, M.~Homer, and R.~Huerta, ``On the
  calibration of sensor arrays for pattern recognition using the minimal number
  of experiments,'' \emph{Chemometrics and Intelligent Laboratory Systems},
  vol. 130, pp. 123--134, 2014.

\bibitem{de-vito-airquality}
S.~De~Vito, E.~Massera, M.~Piga, L.~Martinotto, and G.~Di~Francia, ``On field
  calibration of an electronic nose for benzene estimation in an urban
  pollution monitoring scenario,'' \emph{Sensors and Actuators B: Chemical},
  vol. 129, no.~2, pp. 750--757, 2008.

\bibitem{uci-tetouan}
A.~Salam and A.~El~Hibaoui, ``Power consumption of tetouan city,'' UCI Machine
  Learning Repository, 2018, [Dataset].

\end{thebibliography}

\end{document}